%% file: mu_main.tex
\documentclass[10pt]{article}

\input{arxiv_style}

\input{macros}

\input{project_macros}

\usepackage{color-edits} 
\addauthor{gl}{red}
\addauthor{cd}{blue}
\usepackage{enumitem} 
\usepackage{flafter}
\algrenewcommand\algorithmicrequire{\textbf{Input:}} 
\usepackage{cleveref}
\setcitestyle{numbers,square}

\title{The Sample Complexity of Policy Learning with $\mu$-Resets}
\author{Gene Li \\ \texttt{gene@ttic.edu}}
\date{\today}

\begin{document}
\maketitle

\begin{abstract}
We study policy-based reinforcement learning under the $\mu$-resets interaction protocol of Kakade and Langford~\cite{KL02}. This interaction protocol enables the learner to sample trajectories from a given exploratory reset distribution $\mu$, in addition to the starting distribution. We resolve the question raised by \cite{KLS25} on the role of policy realizability for the sample complexity of this problem. Critically, the dependence on horizon $H$ is governed by the notion of coverage assumed of the reset distribution. Under bounded all-policy concentrability, we show a $\exp\prn{\Omega(H)}$ sample complexity lower bound; with bounded pushforward concentrability, we show the dependence on horizon is tightly characterized as $\exp(\Theta(\sqrt H))$.
\end{abstract}

\section{Introduction}
Kakade and Langford's influential paper \cite{KL02} introduced the $\mu$-resets interaction protocol for reinforcement learning (RL).
\begin{definition}[$\mu$-Resets]
    The learner is given online sampling
    access to an exploratory reset distribution
    $\mu$ over states.  The learner can either roll out trajectories from the
    initial state distribution or from the reset distribution $\mu$.
\end{definition}
In this paper, we study the sample complexity of policy learning under the $\mu$-resets interaction protocol: given a policy class $\Pi$, how many trajectories does one need to collect in order to find a near-optimal policy $\widehat{\pi}$ satisfying $V^{\estpi} \ge \max_{\pi \in \Pi} V^\pi - \eps$?

Two classical algorithms have been proposed for this setting: Conservative
Policy Iteration \cite{KL02} for discounted, infinite-horizon RL and Policy
Search by Dynamic Programming (PSDP) \cite{BKSN03} for finite-horizon RL.
Both require a particularly stringent assumption on the representational
capacity of $\Pi$ called \emph{policy completeness}. Under policy
completeness, their sample complexities are polynomial in the relevant problem
parameters. Informally, policy completeness requires that the policy class is closed under the policy improvement operator. A priori, it is unclear if this assumption is fundamentally necessary, or just needed to get the analysis to work.

\textit{Quick aside on notation.} In the remainder of the paper, we focus on episodic, finite-horizon RL over MDPs $M=(\X,\A, H, P,R,d_1)$ with (large, but finite) layered state space $\X=\bigsqcup_{h=1}^H\X_h$, action set $\A$ with cardinality $\abs{\A}=A$, horizon $H$, transition function $P$, bounded rewards $R$, and initial state distribution $d_1$. Hence, the exploratory reset distribution $\mu = \crl{\mu_h}_{h=1}^H$ consists of distributions $\mu_h \in \Delta(\X_h)$. We assume $\mu_1 = d_1$ (only making the learner's job harder); in every round of interaction, the learner picks a reset layer $h \in [H]$ and roll out a trajectory from $\mu_h$. We assume the cumulative reward $\sum_h r_h$ in every trajectory is bounded in $[0,1]$.

\cite{KLS25} study the sample complexity of policy learning with $\mu$-resets without the assumption of policy completeness. In the agnostic policy learning setting, they show an information-theoretic sample complexity lower bound of $\exp(\Omega(H))$. However, they left open the setting where the policy class satisfies \emph{realizability} (that the optimal policy $\pi^\star \in \Pi$). Policy realizability is implied by policy completeness but is a much weaker (and arguably more natural) assumption. Unfortunately, the aforementioned $\exp(\Omega(H))$ lower bound crucially requires the non-realizability of $\Pi$. \cite{KLS25} additionally show that with bounded pushforward concentrability (to be defined below), PSDP achieves $\exp(\Theta(H))$ sample complexity under policy realizability via a new upper bound analysis and algorithm-dependent lower bound for PSDP. In an information-theoretic sense, the sample complexity of RL under $\mu$-resets with policy realizability has remained completely open.

\paragraph{Our results.} We characterize the sample complexity of RL under $\mu$-resets when the policy class is realizable. We first list three well-studied coverage conditions which can be used to characterize the quality of the reset distribution $\mu$.\footnote{Throughout we use the convention that $1/0 = \infty$.} In what follows, we use $d^\pi_h(\cdot) \in \Delta(\X_h)$ to denote the policy occupancy measure.

\begin{definition}[Policy-class concentrability]
The policy-class concentrability coefficient is
\[
\cconc(\mu;\Pi,M)
:=\max_{\pi\in\Pi}\max_{h\in[H]}
  \nrm*{\frac{d_h^\pi}{\mu_h}}_\infty.
\]
\end{definition}

\begin{definition}[All-policy concentrability]
Let $\Pi_{\mathrm{all}}$ denote the set of all Markov policies on $M$.  The
all-policy concentrability coefficient is
\[
\call(\mu;M)
:=\max_{\pi\in\Pi_{\mathrm{all}}}\max_{h\in[H]}
  \nrm*{\frac{d_h^\pi}{\mu_h}}_\infty.
\]
\end{definition}

\begin{definition}[Pushforward concentrability]
The pushforward concentrability coefficient is
\[
\cpush(\mu;M)
:=\max_{h\in[H-1]}\max_{x\in\X_h,\,a\in\A}
  \nrm*{\frac{P(\cdot\mid x,a)}{\mu_{h+1}}}_\infty.
\]
\end{definition}

It is straightforward from the definitions to show that
$\cconc(\mu;\Pi,M)\le\call(\mu;M)\le\cpush(\mu;M)$. When clear from the
context, we write $\cconc$, $\call$, and $\cpush$.

Our main results are stated below and summarized in \Cref{tab:results}.

\begin{table}[t]
\centering
{\normalsize
\renewcommand{\arraystretch}{1.35}
\begin{tabular}{l c c c c c}
    & $\cconc$
    & & $\call$
    & & $\cpush$ \\
    \hline
    Policy Completeness
    & \makecell{\textcolor{Green}{\cmark}\\\cite{BKSN03},\\see also \cite[Thm.~1 of][]{KLS25}}
    & $\Longrightarrow$
    & \textcolor{Green}{\cmark}
    & $\Longrightarrow$
    & \textcolor{Green}{\cmark}
    \\[0.7em]
    Policy Realizability
    & \textcolor{red}{\xmark}
    & $\Longleftarrow$
    & \makecell{\textcolor{red}{\xmark}\\Thm.~\ref{thm:all-policy-lower-bound}}
    &
    & \makecell{\textcolor{Purple}{\xmark$^\star$}\\Thms.~\ref{thm:pushforward-upper-bound}--\ref{thm:pushforward-lower-bound}}
    \\[0.7em]
    Agnostic
    & \textcolor{red}{\xmark}
    & $\Longleftarrow$
    & \textcolor{red}{\xmark}
    & $\Longleftarrow$
    & \makecell{\textcolor{red}{\xmark}\\\cite[Thm.~3 of][]{KLS25}}
\end{tabular}}
\caption{Sample complexity under $\mu$-resets. Rows denote the different representational assumptions on $\Pi$; columns denote coverage assumptions on the reset $\mu$. A \textcolor{Green}{\cmark} denotes $\poly(C,A, H\log\abs{\Pi},\eps^{-1})$ sample complexity (where $C \in \crl{\ccov, \call, \cpush}$), while \textcolor{red}{\xmark} denotes a $\exp(\Omega(H))$ lower bound. The starred entry is instead $\exp(\Theta(\sqrt H))$ for constant $A$ and $\cpush$. Arrows show implications of upper and lower bounds to weaker statements.}
\label{tab:results}
\end{table}

\begin{theorem}[Lower bound for $\call$]\label{thm:all-policy-lower-bound}
For any sufficiently large $H \in \mathbb{N}$, there exists a policy class $\Pi$ of size $2^H$; a family of MDPs $\cM$ over a state space of size $2^{O(H)}$, binary action space, and horizon $H$; and known reset distribution $\mu$; such that every $M \in \cM$ satisfies (i) $\Pi$ is realizable, and (ii) $\mu$ satisfies $\call(\mu;M) \le 4$. Any proper deterministic algorithm that returns a $1/8$-optimal policy with constant probability must sample at least $2^{\Omega(H)}$ trajectories for some MDP in $\cM$.
\end{theorem}

Thus, if we assume the reset has bounded $\call$, we still require the stringent policy completeness assumption in order to get sample-efficient learning, and realizability doesn't buy us anything.

Next, we turn to the assumption of bounded pushforward concentrability of the reset $\mu$.

\begin{theorem}[Upper bound for $\cpush$]\label{thm:pushforward-upper-bound}
Let $M$ be an MDP and suppose $\Pi$ is a realizable class of policies. Furthermore, suppose the reset distributions $\mu$ satisfy pushforward concentrability with parameter $\cpush$. Then for every $\eps, \delta \in(0,1)$, there is an
algorithm using $\mu$-resets that, with probability at least $1-\delta$, returns an $\eps$-optimal policy
$\widehat{\pi}$ using at most
\[
O\left(
    \frac{(A\cpush)^{2\sqrt H}H^{3/2}}{\eps^2}
    \log\frac{H\abs{\Pi}}{\delta}
\right)
\quad \mathrm{trajectories.}
\]
\end{theorem}

\begin{theorem}[Lower bound for $\cpush$]\label{thm:pushforward-lower-bound}
For any sufficiently large $H \in \mathbb{N}$, there exists a policy class $\Pi$ of size $2^{\sqrt H}$; a family of MDPs $\cM$ over a state space of size $2^{O(\sqrt H)}$, binary action space, and horizon $H$; and a known reset distribution $\mu$; such that every $M \in \cM$ satisfies (i) $\Pi$ is realizable, and (ii) $\mu$ satisfies $\cpush(\mu;M)\le 8$. Any proper deterministic algorithm that returns a $1/8$-optimal policy with constant probability must sample at least $2^{\Omega(\sqrt H)}$ trajectories for some MDP in $\cM$.
\end{theorem}

\pref{thm:pushforward-lower-bound} can be viewed as an information-theoretic strengthening of the algorithmic lower bound of \cite[Theorem 8, ][]{KLS25} that additionally builds in $\sqrt{H}$-length combination locks into the construction.

Although these bounds remain exponential, together \pref{thm:pushforward-upper-bound} and \ref{thm:pushforward-lower-bound} show that bounded pushforward concentrability changes the horizon dependence from $\exp(\Theta(H))$ to $\exp(\Theta(\sqrt H))$. This comparison has three consequences. First, realizability yields an exponential improvement over the agnostic setting, where \cite[Theorem~3 of][]{KLS25} prove a $2^{\Omega(H)}$ lower bound. Second, PSDP requires $\exp(\Theta(H))$ samples under the same realizability and pushforward assumptions~\cite[Theorems~6 and 8 of][]{KLS25}, and is therefore suboptimal. Third, our results give a quantitative separation between all-policy and pushforward concentrability, analogous to separations known in offline RL~\cite{FKSX22,JRSW24,XJ21}.

\paragraph{Statement of AI Use.} Results were obtained via GPT 5.6 Pro, with assistance from Charlie Hou and Xuchen You. We used GPT 5.6 Sol Codex to aid in writing and presentation.

\section{Lower Bound under All-Policy Concentrability}
\label{sec:all-policy-lower-bound}

In this section, we sketch the proof of
\pref{thm:all-policy-lower-bound}. We first describe the construction and
then explain why reset access does not help. We omit the formal
information-theoretic arguments, which follow the standard analysis for rich-observation combination
locks~\cite{SDMMS21,JLRSS23,KLS25}.

\paragraph{Lower bound construction.}
We use a variant of the rich-observation combination lock, which is essentially a Block MDP with enormous decoder class. See
\pref{fig:reset-poisoning} for an illustration. For simplicity, assume that $H$ is even. We let $\actionsp=\{0,1\}$ and define the policy class to be open-loop policies:
\[
    \Pi
    =
    \left\{
        \pi:
        \pi_h(x)\equiv a_h\ \text{for every }x\in\X_h,\ 
        (a_1,\ldots,a_H)\in\{0,1\}^H
    \right\}, \quad \text{thus, }\abs{\Pi}=2^H.
\]
Each MDP instance is parameterized by an optimal policy
$\pi^\star\in\Pi$. Its first $H/2$ actions determine the value from the
initial state; its actions from layer $H/2+1$ through layer $H-1$ serve only
to mask information in reset trajectories. The final action is irrelevant.

Every layer has four latent states: a good state $s_h^{\sgood}$, a
bad state $s_h^{\sbad}$, a verifier state $s_h^{\sver}$, and a neutral state
$s_h^{\snull}$. The initial latent state is $s_1^{\sgood}$.
In the first half of the MDP, the good--bad and verifier--neutral pairs form two combination locks.
The optimal action keeps the learner on the good or verifier track, whereas an
incorrect action transits the learner from $s_h^{\sgood}$ to
$s_{h+1}^{\sbad}$ and from $s_h^{\sver}$ to $s_{h+1}^{\snull}$. In the
second half, the good/neutral states advance under either action, while the
verifier and bad states form a new ``poison'' combination lock.
Formally, for every $h\in[H-1]$, the latent dynamics $\optlatp$ are
\[
\begin{aligned}
    \optlatp(\cdot\mid s_h^{\sgood},a)
    &=
    \begin{cases}
        \delta_{s_{h+1}^{\sgood}}
            &  a = \pi_h^\star \text{ and } h \le H/2,\\
        \delta_{s_{h+1}^{\sgood}}
            & h>H/2,\\
        \delta_{s_{h+1}^{\sbad}}
            & \text{otherwise},
    \end{cases}
    &\qquad
    \optlatp(\cdot\mid s_h^{\sbad},a)
    &=
    \delta_{s_{h+1}^{\sbad}},
    \\[2mm]
    \optlatp(\cdot\mid s_h^{\sver},a)
    &=
    \begin{cases}
        \delta_{s_{h+1}^{\sver}}
            & a=\pi_h^\star,\\
        \delta_{s_{h+1}^{\snull}}
            & a\ne\pi_h^\star \text{ and } h\le H/2,\\
        \delta_{s_{h+1}^{\sbad}}
            & a\ne\pi_h^\star \text{ and } h>H/2,
    \end{cases}
    &
    \optlatp(\cdot\mid s_h^{\snull},a)
    &=
    \delta_{s_{h+1}^{\snull}}.
\end{aligned}
\]

\begin{figure}[t]
\centering
\includegraphics[width=\textwidth]{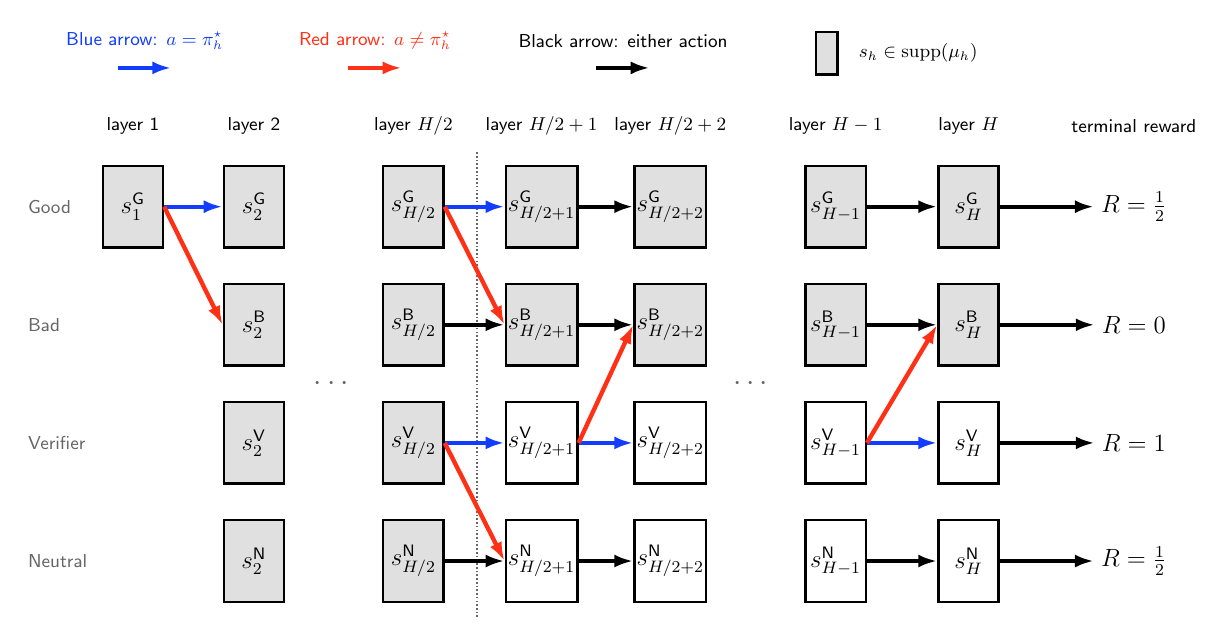}
\caption{Construction used for proof of Theorem~\ref{thm:all-policy-lower-bound}.}
\label{fig:reset-poisoning}
\end{figure}

Rewards are nonzero only in the final layer:
\[
    \optlatr(s,a)
    =
    \begin{cases}
        1
            & s=s_H^{\sver},\\
        \frac12
            & s\in\{s_H^{\sgood},s_H^{\snull}\},\\
        0
            & \text{otherwise}.
    \end{cases}
\]
With the latent dynamics fully specified, we can define the set of Block MDP instances which are parameterized by optimal policy $\optpi\in\Pi$ and decoder $\optdec\in\Phi$, as $\cM=\{M_{\optpi,\optdec}:\optpi\in\Pi,\ \optdec\in\Phi\}$. To define the decoder class, we use the standard idea from \cite{SDMMS21, KLS25}: for every layer let $\X_h$ be a large (observed) state space of size $m = 2^{cH}$ for some sufficiently large $c>0$ and define $\Phi$ to be the set of all possible decoders $\optdec: \X\to\latentsp$ which assign equal number of (observed) states to every latent state. A transition to $s \in \latentsp_h$ emits an observation according to the emission function $\emission(s) =\unif\bigl(\{x\in\X_h:\optdec(x)=s\}\bigr)$.

Finally, the reset distribution $\mu$ is defined as:
\[
    \mu_{h,\optdec}
    =
    \begin{cases}
        \emission(s_1^{\sgood}),
            & h=1,\\
        \frac14\emission(s_h^{\sgood})+\frac14\emission(s_h^{\sbad})+\frac14\emission(s_h^{\sver})+\frac14\emission(s_h^{\snull}),
            & 2\le h\le H/2,\\
        \frac12\emission(s_h^{\sgood})
        +\frac12\emission(s_h^{\sbad}),
            & H/2<h\le H.
    \end{cases}
\]
That is, the first state is emitted from the good state; for the other layers in the first half, the reset is uniform over all four latent states; the second-half resets are uniform over the good and bad states. It is clear that for every $\optpi \in \Pi$ and $\optdec \in \Phi$, $$\call(\mu_{\optdec};M_{\optpi,\optdec})\le 4.$$ Lastly, one can also check that $\optpi$ is optimal under these dynamics from any state, and it achieves value of $1/2$ from the starting distribution $\mu_1$. 

\paragraph{Why doesn't reset access allow sample-efficient learning of the optimal policy?}
Fix an instance $M_{\optpi,\optdec}$. The large decoder class $\Phi$ has statistical complexity $\log \abs{\Phi}$ which is exponential in $H$. With high probability, every observed state is a fresh, nonrepeated observation. Thus, transition data from $M_{\optpi,\optdec}$ leaks essentially no information about the underlying $\optpi$, and the learner is forced to learn $\optpi$ from reward observations.

Observe that a trajectory from the initial distribution only achieves positive reward iff the first $H/2$ actions exactly match $\pi^\star$, thus requiring $2^{\Omega(H)}$ samples to learn. Therefore, the learner must also try to utilize the trajectories rolled out from resets. We next argue that this is also doomed to fail:
\begin{itemize}
    \item Rolling out from resets $h > H/2$ are useless - no matter what policy the learner plays, they will see rewards of $1/2$ and $0$ equally often.
    \item Rolling out from resets $2 \le h \le H/2$ is a more delicate case. Fix any partial policy $\pi_{h:H-1} \in \crl{0,1}^{H-h}$, and let $Z$ be the random variable representing the reward we observe by sampling a state $\unif(\X_h)$ then rolling out with $\pi_{h:H-1}$. Then if $\pi_{h:H} \ne \pi^\star_{h:H}$, then $Z$ is $1/2$ or $0$ with equal probability. Also, if $\pi_{h:H} = \pi^\star_{h:H}$, then $Z$ is 1 with probability $1/4$, $1/2$ with probability $1/2$, and $0$ with probability $1/4$.

    This calculation elucidates the role of the ``poison'' combination lock: it effectively prevents the learner from just doing backwards induction from layer $H/2$ because this will only work if the learner already knows poison suffix $\pi^\star_{H/2:H-1}$ (which itself can only be randomly guessed with exponentially small probability).
\end{itemize}
Thus, together with the initial distribution argument, we have argued that there is no way for the learner to use $2^{O(H)}$ samples to properly identify $\optpi$ given a random instance from $\cM$.

\section{Upper Bound under Pushforward Concentrability}
\label{sec:pushforward-upper-bound}
In this section, we prove our main upper bound using \textsf{BlockPSDP}, a
blockwise variant of PSDP whose pseudocode is given in
\pref{alg:block-psdp}.

\paragraph{Algorithm.}
Fix a number of blocks $K\in[H]$ and let $L=H/K$. We may assume that $L$ is
an integer; otherwise, we can pad the last block. For each $k\in[K]$, set
\[
    \mathsf{start}_k=1+(k-1)L,
    \qquad
    \mathsf{Block}_k
    =
    \{\mathsf{start}_k,\ldots,kL\}.
\]
The algorithm works backward over these
blocks. At block $k$, it explores uniformly until the end of the block
and then follows the suffix already learned on later blocks. Here $\circ$
denotes layerwise concatenation of partial policies. We write $\Pi_h$ for the
restriction of $\Pi$ to layer $h$, and $\Pi_{\mathsf{Block}_k}$ for its
restriction to $\mathsf{Block}_k$.

\begin{algorithm}[t]
\caption{\textsf{BlockPSDP}}
\label{alg:block-psdp}
\begin{algorithmic}[1]
\Require Reset distributions $\mu=\{\mu_h\}_{h\in[H]}$, policy class
$\Pi$, number of blocks $K$, and sample size $n$.
\For{$k=K,\ldots,1$}
    \State Initialize dataset $\cD_k=\varnothing$.
    \For{$n$ times}
        \State Sample $x_{\mathsf{start}_k}\sim
        \mu_{\mathsf{start}_k}$ and
        $a_{\mathsf{Block}_k}\sim\unif(\A^L)$, generating
        $\tau_k=(x_h,a_h)_{h\in\mathsf{Block}_k}$.
        \State Let $v_k\coloneqq\sum_{h=\mathsf{start}_k}^H r_h$ be the
        return from executing
        $a_{\mathsf{Block}_k}\circ
        \widehat\pi_{\mathsf{start}_{k+1}:H}$ from
        $x_{\mathsf{start}_k}$.
        \State Set $\cD_k\gets\cD_k\cup\{(\tau_k,v_k)\}$.
    \EndFor
    \State Call policy optimization oracle:
    $\widehat\pi_{\mathsf{Block}_k}
    \in\displaystyle\argmax_{\pi_{\mathsf{Block}_k}
    \in\Pi_{\mathsf{Block}_k}}
    \frac{A^L}{n}\sum_{(\tau_k,v_k)\in\cD_k}
    \ind{a_h=\pi_h(x_h)\ \forall h\in\mathsf{Block}_k}\,v_k$.
    \State Set
    $\widehat\pi_{\mathsf{start}_k:H}
    \gets
    \widehat\pi_{\mathsf{Block}_k}
    \circ\widehat\pi_{\mathsf{start}_{k+1}:H}$.
\EndFor
\State \Return $\widehat\pi_{1:H}$.
\end{algorithmic}
\end{algorithm}

Recall that PSDP learns the policy via resets one layer at a time. Under
pushforward concentrability, its error can amplify by a factor of $\cpush$
at each layer~\cite{KLS25}. The key idea of \textsf{BlockPSDP} is to learn
$L$ consecutive layers at once by importance sampling trajectories that
explore uniformly within a block. This costs $A^L$ in sample complexity,
but error amplification occurs only across the $K-1$ boundaries between
the $K=H/L$ blocks.
Balancing these two terms gives the desired sample complexity.

\paragraph{Proof of \pref{thm:pushforward-upper-bound}.}
The analysis has two steps: estimation within each block and error
propagation across block boundaries.

\paragraph{Importance sampling guarantee.}
Fix a block $k$, and suppose we have constructed
$\widehat\pi_{\mathsf{start}_{k+1}:H}$. For
$\pi_{\mathsf{Block}_k}\in\Pi_{\mathsf{Block}_k}$, define
\[
    J_k(\pi_{\mathsf{Block}_k})
    =
    \E_{x\sim\mu_{\mathsf{start}_k}}
    \left[
        V_{\mathsf{start}_k}^{
            \,\pi_{\mathsf{Block}_k}
            \circ\widehat\pi_{\mathsf{start}_{k+1}:H}
        }(x)
    \right].
\]
The empirical objective in
\pref{alg:block-psdp} is the standard trajectory importance sampling
estimate for $J_k(\pi_{\mathsf{Block}_k})$. Standard concentration for
importance sampling~\cite{JLRSS23} implies that for any $\eta\in(0,1]$, as long as 
\begin{equation}
    n
    \ge
    c\frac{A^L}{\eta^2}
    \log\frac{2\abs{\Pi}}{\delta},
    \label{eq:block-samples}
\end{equation}
for some sufficiently large $c > 0$, then, with probability at least $1-\delta$,
\begin{equation}
    J_k\bigl(\widehat\pi_{\mathsf{Block}_k}\bigr)
    \ge
    J_k\bigl(\pi^\star_{\mathsf{Block}_k}\bigr)-\eta.
    \label{eq:block-local}
\end{equation}

\begin{lemma}[Error propagation]
\label{lem:block-propagation}
For $k\in[K]$, let
\[
    e_k
    =
    \E_{x\sim\mu_{\mathsf{start}_k}}
    \left[
        V_{\mathsf{start}_k}^{\pi^\star}(x)
        -
        V_{\mathsf{start}_k}^{\widehat\pi_{\mathsf{start}_k:H}}(x)
    \right].
\]
On the event \eqref{eq:block-local},
\begin{equation}
    e_K\le\eta,
    \qquad
    e_k\le\eta+\cpush e_{k+1}
    \quad\text{for }k<K.
    \label{eq:block-recursion}
\end{equation}
\end{lemma}

\begin{proof}
For $k<K$, let $\nu_{k+1}$ denote the distribution of
$x_{\mathsf{start}_{k+1}}$ obtained by drawing
$x_{\mathsf{start}_k}\sim\mu_{\mathsf{start}_k}$ and executing
$\pi^\star_{\mathsf{Block}_k}$. Then
\begin{align*}
    e_k
    &=
    J_k\bigl(\pi^\star_{\mathsf{Block}_k}\bigr)
    -
    J_k\bigl(\widehat\pi_{\mathsf{Block}_k}\bigr)
    +
    \underbrace{
        \E_{x\sim\nu_{k+1}}
        \left[
            V_{\mathsf{start}_{k+1}}^{\pi^\star}(x)
            -
            V_{\mathsf{start}_{k+1}}^{
                \widehat\pi_{\mathsf{start}_{k+1}:H}
            }(x)
        \right]
    }_{\text{continuation gap}} \\
    &\le
    \eta
    +
    \E_{x\sim\nu_{k+1}}
    \left[
        V_{\mathsf{start}_{k+1}}^{\pi^\star}(x)
        -
        V_{\mathsf{start}_{k+1}}^{
            \widehat\pi_{\mathsf{start}_{k+1}:H}
        }(x)
    \right] \\
    &\le
    \eta+\cpush e_{k+1}.
\end{align*}
The last line uses pushforward concentrability as well as the optimality of $\pi^\star$, which ensures that the continuation gap between policies is pointwise nonnegative. For the last block there is no continuation gap, so we have
\[
    e_K
    =
    J_K\bigl(\pi^\star_{\mathsf{Block}_K}\bigr)
    -
    J_K\bigl(\widehat\pi_{\mathsf{Block}_K}\bigr)
    \le \eta.
\]
\end{proof}

\paragraph{Final guarantee.} At every block let the number of rollouts be
\[
    n
    =
    O\left(
        \frac{A^L\cpush^{2(K-1)}K^2}{\eps^2}
        \log\frac{2\abs{\Pi}}{\delta}
    \right).
\] 
Then by the importance sampling guarantee, \pref{lem:block-propagation}, and a union bound, with probability at least
$1-K\delta$, $V^{\pi^\star}-V^{\widehat\pi_{1:H}} = e_1 \le \eps$. For simplicity, take $K=L=\sqrt H$ and assume these quantities are
integers. Replacing $\delta$ by $\delta/K$ gives success probability at
least $1-\delta$ and
\[
    nK
    =
    O\left(
        \frac{(A\cpush)^{2\sqrt H}H^{3/2}}{\eps^2}
        \log\frac{H\abs{\Pi}}{\delta}
    \right).
\]
This completes the proof of \pref{thm:pushforward-upper-bound}.\qedhere

\section{Lower Bound under Pushforward Concentrability}
\label{sec:pushforward-lower-bound}

In this section, we sketch the proof of
\pref{thm:pushforward-lower-bound}. We first describe the construction and
then explain why reset access does not help. As in
\pref{sec:all-policy-lower-bound}, we omit the formal information-theoretic
argument, which follows from a standard random-decoder analysis and the
reward identity below. We use $K$ blocks of length $L$, so that $H=KL$.
For the lower bound, take $K=L=\sqrt H$ and assume these quantities are
integers.

\paragraph{Lower bound construction.}
We use a recursive rich-observation combination lock; see
\pref{fig:recursive-reset-poisoning} for an illustration. Again let
$\actionsp=\{0,1\}$. Partition the horizon into
$\mathsf{Block}_1,\ldots,\mathsf{Block}_K$, where
\[
    \mathsf{Block}_k
    :=
    \{(k-1)L+1,\ldots,kL\},
    \qquad k\in[K].
\]
An MDP instance is parameterized by a hidden key
$\theta=(\theta_1,\ldots,\theta_L)\in\{0,1\}^L$. The policy class consists
of open-loop policies which repeat the same length-$L$ action sequence in
every block. Namely, for any $\vartheta\in\{0,1\}^L$, define the policy $\pi^\vartheta$ as
\[
    \pi_h^\vartheta(x)=\vartheta_\ell
    \qquad
    \text{when }h=(k-1)L+\ell,
\]
and let $\Pi=\{\pi^\vartheta:\vartheta\in\{0,1\}^L\}$. Thus,
$\abs{\Pi}=2^L=2^{\sqrt H}$. We design the MDP instance so that
$\pi^\star=\pi^\theta$.

Every layer has four latent states: a real state
$s_h^{\mathsf R}$, a decoy state $s_h^{\mathsf D}$, a one state $s_h^1$,
and a zero state $s_h^0$. The initial latent state is $s_1^{\mathsf R}$. For $p\in[0,1]$, define
\[
    \mathsf{Collect}_{h+1}(p)
    :=
    p\delta_{s_{h+1}^1}
    +(1-p)\delta_{s_{h+1}^0},
    \qquad
    p_k:=4^{-k},
    \quad k\in[K].
\]
We now describe the dynamics. Fix
$h=(k-1)L+\ell\in[H-1]$.
\begin{itemize}
    \item The one and zero states progress to their respective next states
    under either action:
    \[
        \optlatp(\cdot\mid s_h^1,a)=\delta_{s_{h+1}^1},
        \qquad
        \optlatp(\cdot\mid s_h^0,a)=\delta_{s_{h+1}^0}.
    \]

    \item In every block, the real--zero pair forms a combination lock.
    From a real state, the correct action $\theta_\ell$ advances to the next
    real state, while a wrong action moves to the zero state. At the end of
    a nonfinal block $k$, the correct action transitions to
    $\mathsf{Collect}(p_k)$. Thus,
    \[
        \optlatp(\cdot\mid s_h^{\mathsf R},a)
        =
        \begin{cases}
            \delta_{s_{h+1}^{\mathsf R}}
                & a=\theta_\ell,\ \ell<L,\\
            \mathsf{Collect}_{h+1}(p_k)
                & k<K,\ a=\theta_L,\ \ell=L,\\
            \delta_{s_{h+1}^0}
                & \text{otherwise}.
        \end{cases}
    \]

    \item The decoy states prevent the learner from discovering the key by
    working in the opposite direction. In a nonfinal block $k$, a wrong
    action transitions to $\mathsf{Collect}(p_{k+1})$. Correct actions
    advance along the decoy track, and completing the block moves to a real
    state at the start of block $k+1$. In the final block, either action
    moves to the zero state. Formally,
    \[
        \optlatp(\cdot\mid s_h^{\mathsf D},a)
        =
        \begin{cases}
            \mathsf{Collect}_{h+1}(p_{k+1})
                & k<K,\ a\ne\theta_\ell,\\
            \delta_{s_{h+1}^{\mathsf D}}
                & k<K,\ a=\theta_\ell,\ \ell<L,\\
            \delta_{s_{h+1}^{\mathsf R}}
                & k<K,\ a=\theta_L,\ \ell=L,\\
            \delta_{s_{h+1}^0}
                & k=K.
        \end{cases}
    \]
\end{itemize}
All rewards before layer $H$ are zero. At layer $H$, the two nonzero
reward distributions are
\[
    R_H(s,a)
    =
    \begin{cases}
        1
            & s=s_H^1,\\
        \ber(p_K)
            & s=s_H^{\mathsf R},\ a=\theta_L.
    \end{cases}
\]
Every other terminal state-action pair receives reward zero.

We realize these dynamics as a rich-observation Block MDP using the same
decoder and emission construction as in 
\pref{sec:all-policy-lower-bound}. Let $m=2^{c\sqrt H}$ for a sufficiently
large constant $c>0$. For each $h\ge2$, let $\abs{\X_h}=8m$, and let the
decoder $\phi$ partition $\X_h$ into the preimages of
$s_h^{\mathsf R},s_h^{\mathsf D},s_h^0,s_h^1$ with respective sizes
$m,4m,2m,m$. A transition to a latent state emits uniformly from its
preimage. This defines the family $M_{\theta,\phi}$, indexed by the hidden
key $\theta$ and the decoder $\phi$.

We take $\mu_1$ to be the initial distribution and
$\mu_h=\unif(\X_h)$ for every $h\ge2$. If $X\sim\mu_h$ and
$S=\phi(X)$, then
\[
    \Pr(S=s_h^{\mathsf R})=\frac18,\qquad
    \Pr(S=s_h^{\mathsf D})=\frac12,\qquad
    \Pr(S=s_h^0)=\frac14,\qquad
    \Pr(S=s_h^1)=\frac18.
\]

\begin{figure}[t]
\centering
\includegraphics[width=\textwidth]
    {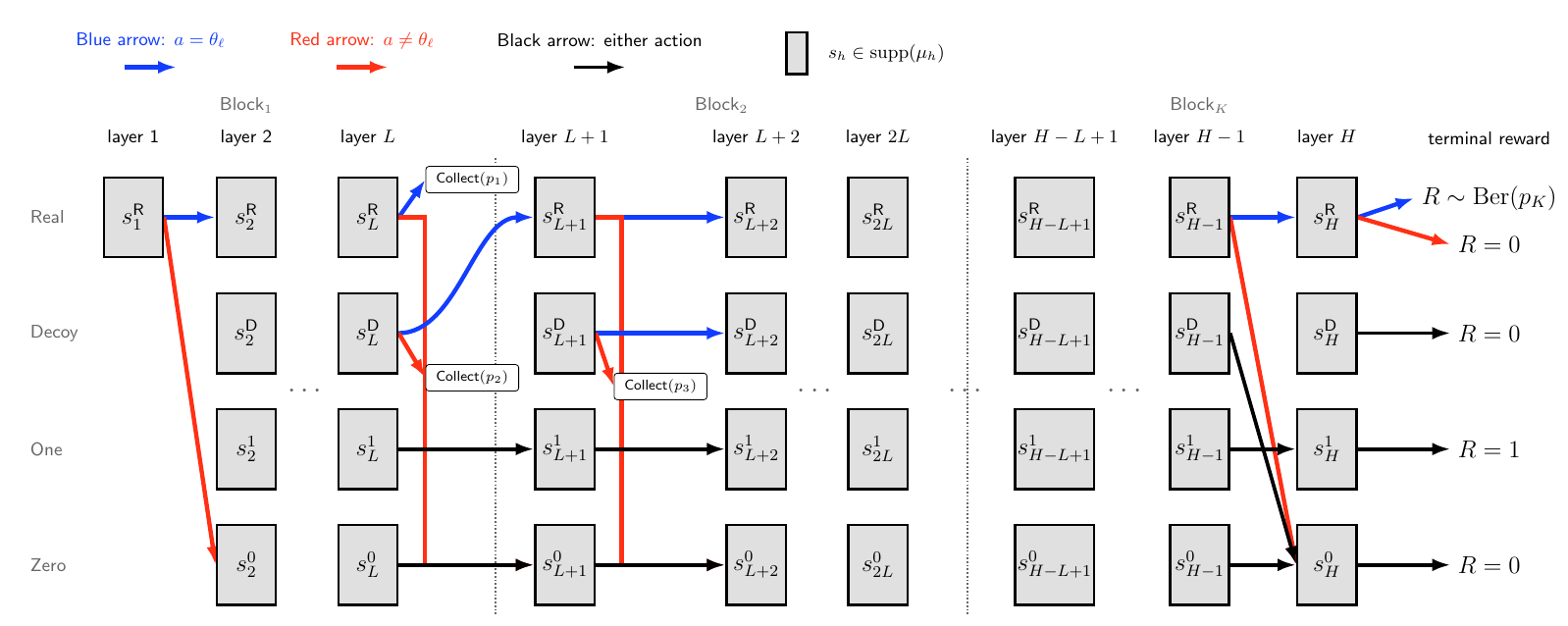}
\caption{Construction used for proof of
Theorem~\ref{thm:pushforward-lower-bound}.}
\label{fig:recursive-reset-poisoning}
\end{figure}

From this, it is immediate that every instance satisfies
\[
    \cpush(\mu;M_{\theta,\phi})\le8.
\]
Also, $\pi^\theta$ is optimal from every state. At a real state, it plays
the unique rewarding suffix. At a decoy state in a nonfinal block $k$,
it completes the current suffix and the next block, obtaining value
$p_{k+1}$. A deviation while the process remains on the decoy track
immediately transitions to $\mathsf{Collect}(p_{k+1})$, while completing
the current block leads to a real state whose value is at most
$p_{k+1}$. In the final block, every policy has value zero from a decoy
state, and actions do not affect the value from the one and zero states.
Consequently,
\[
    V^{\pi^\theta}=p_1=\frac14,
    \qquad
    V^{\pi^\vartheta}=0
    \quad\text{for every }\vartheta\ne\theta.
\]

\paragraph{Why doesn't reset access allow sample-efficient learning of the optimal policy?}
Fix an instance $M_{\theta,\phi}$. As in
\pref{sec:all-policy-lower-bound}, the large random decoder ensures that
transition observations leak essentially no information about $\theta$, so
the learner is forced to use reward observations.

Similar to the lower bound in \pref{sec:all-policy-lower-bound}, online
trajectories from the initial state distribution $\mu_1$ require
$2^{\Omega(L)}$ samples to guess the key $\theta$. We next explain why
reset trajectories do not allow the learner to recover the key one bit at
a time.

Fix a reset layer $h=(k-1)L+\ell\ge2$ in a nonfinal block $k<K$ and an
action sequence through the end of block $k+1$. Let
\[
\begin{aligned}
    I &= \ind{a_{h:kL}=\theta_{\ell:L}},
    &\qquad& \text{(action matches the remaining suffix of
        $\mathsf{Block}_k$)},\\
    J &= \ind{a_{kL+1:(k+1)L}=\theta},
    && \text{(action matches $\theta$ in $\mathsf{Block}_{k+1}$)}.
\end{aligned}
\]
The rewards are either zero or one, so we compute the probability of the reward being 1:
\begin{align*}
    \Pr\prn*{r_H=1}
    &=
    \underbrace{\frac18}_{\text{from }s_h^1}
    +
    \underbrace{\frac{p_k}{8}I}_{\text{from }s_h^{\mathsf R}}
    +
    \underbrace{\frac{p_{k+1}}{2}(1-I+IJ)}
        _{\text{from }s_h^{\mathsf D}}
    \\
    &=
    \frac18+\frac{p_k}{8}(1+IJ),
    &\text{since }p_{k+1}=\frac14p_k.
\end{align*}
Therefore, every action sequence except the one matching the remaining
suffix of $\mathsf{Block}_k$ and all of $\mathsf{Block}_{k+1}$ produces the
same reward law. This is the essential recursive poisoning property:
pushforward coverage places reset mass on the next-block real state, and
the decoy at that block poisons the new signal again at a scale four times
smaller.

What about in the final block? A reset reaches the real state with
probability $1/8$, so the gap between the correct and incorrect final
actions is $p_K/8$. Thus, standard hypothesis testing bounds require
$\Omega(p_K^{-2})=2^{\Omega(K)}$ samples to identify $\theta_L$.

Combining the two sources of information gives a
$2^{\Omega(\min\{K,L\})}$ trajectory lower bound. Taking
$K=L=\sqrt H$ yields the claimed $2^{\Omega(\sqrt H)}$ lower bound.

\end{document}

%% file: arxiv_style.tex
\usepackage{tgpagella}
\usepackage{parskip}

\usepackage[letterpaper, left=1in, right=1in, top=1in, bottom=1in]{geometry}

\usepackage[dvipsnames]{xcolor}
\usepackage[colorlinks=true, linkcolor=blue!70!black, citecolor=blue!70!black]{hyperref}
\usepackage{microtype}

\usepackage{natbib}

\usepackage{amsthm}
\usepackage{mathtools}
\usepackage{amsmath}
\usepackage{bm}
\usepackage{bbm}
\usepackage{amsfonts}
\usepackage{amssymb}

\usepackage{xpatch}
\usepackage{pifont}
\usepackage{array}
\usepackage{booktabs}
\usepackage{floatrow}
\newfloatcommand{capbtabbox}{table}[][\FBwidth]
\usepackage{blindtext}
\usepackage{caption}
\usepackage{subcaption}

\usepackage{algorithm}
\usepackage{algorithmicx}
\usepackage[noend]{algpseudocode}

%% file: macros.tex
\usepackage{bbm}

\allowdisplaybreaks

\usepackage{mathtools}

\DeclarePairedDelimiter{\abs}{\lvert}{\rvert} 

\DeclarePairedDelimiter{\crl}{\{}{\}}
\DeclarePairedDelimiter{\prn}{(}{)}
\DeclarePairedDelimiter{\nrm}{\|}{\|}

\let\Pr\undefined
 
\DeclareMathOperator{\En}{\mathbb{E}}

\DeclareMathOperator{\Pr}{\bbP}

\DeclareMathOperator*{\argmax}{argmax}

\newcommand{\ind}[1]{\mathbbm{1}\crl*{#1}}    
\newcommand{\eps}{\varepsilon}

\newcommand{\wh}[1]{\widehat{#1}}

\def\ddefloop#1{\ifx\ddefloop#1\else\ddef{#1}\expandafter\ddefloop\fi}
\def\ddef#1{\expandafter\def\csname bb#1\endcsname{\ensuremath{\mathbb{#1}}}}
\ddefloop ABCDEFGHIJKLMNOPQRSTUVWXYZ\ddefloop
\def\ddefloop#1{\ifx\ddefloop#1\else\ddef{#1}\expandafter\ddefloop\fi}
\def\ddef#1{\expandafter\def\csname b#1\endcsname{\ensuremath{\mathbf{#1}}}}
\ddefloop ABCDEFGHIJKLMNOPQRSTUVWXYZ\ddefloop
\def\ddef#1{\expandafter\def\csname c#1\endcsname{\ensuremath{\mathcal{#1}}}}
\ddefloop ABCDEFGHIJKLMNOPQRSTUVWXYZ\ddefloop
\def\ddef#1{\expandafter\def\csname h#1\endcsname{\ensuremath{\widehat{#1}}}}
\ddefloop ABCDEFGHIJKLMNOPQRSTUVWXYZabcdefghijklmnopqrsuvwxyz\ddefloop    
\def\ddef#1{\expandafter\def\csname hc#1\endcsname{\ensuremath{\widehat{\mathcal{#1}}}}}
\ddefloop ABCDEFGHIJKLMNOPQRSTUVWXYZ\ddefloop
\def\ddef#1{\expandafter\def\csname t#1\endcsname{\ensuremath{\widetilde{#1}}}}
\ddefloop ABCDEFGHIJKLMNOPQRSTUVWXYZ\ddefloop
\def\ddef#1{\expandafter\def\csname tc#1\endcsname{\ensuremath{\widetilde{\mathcal{#1}}}}}
\ddefloop ABCDEFGHIJKLMNOPQRSTUVWXYZ\ddefloop

\usepackage{tikz}

\makeatletter
\newtheorem*{rep@theorem}{\rep@title}
\newcommand{\newreptheorem}[2]{%
\newenvironment{rep#1}[1]{%
 \def\rep@title{#2 \ref*{##1}}%
 \begin{rep@theorem}}%
 {\end{rep@theorem}}}
\makeatother

{\newtheorem{lemma}{Lemma}}

{\newtheorem{theorem}{Theorem}}
\newreptheorem{theorem}{Theorem}
\newtheorem{theorem*}{Theorem}
\newtheorem{definition}{Definition}

\usepackage{prettyref}
\newcommand{\pref}[1]{\prettyref{#1}}

\newcommand{\savehyperref}[2]{\texorpdfstring{\hyperref[#1]{#2}}{#2}}
\newrefformat{eq}{\savehyperref{#1}{\textup{(\ref*{#1})}}}
\newrefformat{eqn}{\savehyperref{#1}{Equation~\ref*{#1}}}
\newrefformat{fact}{\savehyperref{#1}{Fact~\ref*{#1}}}
\newrefformat{con}{\savehyperref{#1}{Conjecture~\ref*{#1}}}
\newrefformat{lem}{\savehyperref{#1}{Lemma~\ref*{#1}}}
\newrefformat{def}{\savehyperref{#1}{Definition~\ref*{#1}}}
\newrefformat{line}{\savehyperref{#1}{line~\ref*{#1}}}
\newrefformat{thm}{\savehyperref{#1}{Theorem~\ref*{#1}}}
\newrefformat{corr}{\savehyperref{#1}{Corollary~\ref*{#1}}}
\newrefformat{sec}{\savehyperref{#1}{Section~\ref*{#1}}}
\newrefformat{app}{\savehyperref{#1}{Appendix~\ref*{#1}}}
\newrefformat{ass}{\savehyperref{#1}{Assumption~\ref*{#1}}}
\newrefformat{ex}{\savehyperref{#1}{Example~\ref*{#1}}}
\newrefformat{fig}{\savehyperref{#1}{Figure~\ref*{#1}}}
\newrefformat{alg}{\savehyperref{#1}{Algorithm~\ref*{#1}}}
\newrefformat{rem}{\savehyperref{#1}{Remark~\ref*{#1}}}
\newrefformat{conj}{\savehyperref{#1}{Conjecture~\ref*{#1}}}
\newrefformat{prop}{\savehyperref{#1}{Proposition~\ref*{#1}}}
\newrefformat{proto}{\savehyperref{#1}{Protocol~\ref*{#1}}}
\newrefformat{prob}{\savehyperref{#1}{Problem~\ref*{#1}}}
\newrefformat{claim}{\savehyperref{#1}{Claim~\ref*{#1}}}
\newrefformat{tab}{\savehyperref{#1}{Table~\ref*{#1}}}

\usepackage[shortlabels]{enumitem} 
\usepackage{makecell}
\usepackage{comment}
\usepackage{pifont}
\usepackage{caption}
\usepackage[dvipsnames]{xcolor}
\newcommand{\cmark}{\ding{51}}%
\newcommand{\xmark}{\ding{55}}%
\usepackage{nicefrac} 

\usepackage{tikz} 
\usetikzlibrary{arrows.meta, positioning}

%% file: project_macros.tex
\newcommand{\ccov}{C_\mathsf{cov}}
\newcommand{\cconc}{C_\mathsf{conc}}
\newcommand{\call}{C_\mathsf{all}}

\newcommand{\poly}{\mathrm{poly}}

\providecommand{\X}{\cX}
\providecommand{\A}{\cA}
\providecommand{\E}{\En}

\newcommand{\unif}{\mathrm{Unif}}
\newcommand{\ber}{\mathrm{Ber}}

\newcommand{\cpush}{C_\mathsf{push}}

\newcommand{\estpi}{{\wh{\pi}}}

\newcommand{\optpi}{\pi^\star}

\newcommand{\optlatp}{P_\mathsf{lat}}
\newcommand{\optlatr}{R_\mathsf{lat}}

\newcommand{\optdec}{\phi}

\newcommand{\actionsp}{\cA}
\newcommand{\latentsp}{\cS}
\newcommand{\emission}{\psi}

\newcommand{\sgood}{\mathsf{G}}
\newcommand{\sbad}{\mathsf{B}}
\newcommand{\sver}{\mathsf{V}}
\newcommand{\snull}{\mathsf{N}}


%% file: mu_main.bbl
\begin{thebibliography}{99}

\bibitem[BKSN03]{BKSN03}
J.~A. Bagnell, S.~M. Kakade, J.~G. Schneider, and A.~Y. Ng.
\newblock Policy search by dynamic programming.
\newblock In \emph{Advances in Neural Information Processing Systems 16}, pages 831--838, 2003.

\bibitem[FKSX22]{FKSX22}
D.~J. Foster, A.~Krishnamurthy, D.~Simchi-Levi, and Y.~Xu.
\newblock Offline reinforcement learning: Fundamental barriers for value function approximation.
\newblock In \emph{Proceedings of the 35th Conference on Learning Theory}, volume 178 of \emph{Proceedings of Machine Learning Research}, page 3489, 2022.
\newblock Full version: arXiv:2111.10919.

\bibitem[JLR+23]{JLRSS23}
Z.~Jia, G.~Li, A.~Rakhlin, A.~Sekhari, and N.~Srebro.
\newblock When is agnostic reinforcement learning statistically tractable?
\newblock In \emph{Advances in Neural Information Processing Systems 36}, 2023.

\bibitem[JRSW24]{JRSW24}
Z.~Jia, A.~Rakhlin, A.~Sekhari, and C.-Y.~Wei.
\newblock Offline reinforcement learning: Role of state aggregation and trajectory data.
\newblock In \emph{Proceedings of the 37th Conference on Learning Theory}, volume 247 of \emph{Proceedings of Machine Learning Research}, pages 2644--2719, 2024.

\bibitem[KL02]{KL02}
S.~M. Kakade and J.~Langford.
\newblock Approximately optimal approximate reinforcement learning.
\newblock In \emph{Proceedings of the 19th International Conference on Machine Learning}, pages 267--274, 2002.

\bibitem[KLS25]{KLS25}
A.~Krishnamurthy, G.~Li, and A.~Sekhari.
\newblock The role of environment access in agnostic reinforcement learning.
\newblock In \emph{Proceedings of the 38th Conference on Learning Theory}, volume 291 of \emph{Proceedings of Machine Learning Research}, pages 3405--3406, 2025.
\newblock Full version: arXiv:2504.05405.

\bibitem[SDM+21]{SDMMS21}
A.~Sekhari, C.~Dann, M.~Mohri, Y.~Mansour, and K.~Sridharan.
\newblock Agnostic reinforcement learning with low-rank MDPs and rich observations.
\newblock In \emph{Advances in Neural Information Processing Systems 34}, 2021.

\bibitem[XJ21]{XJ21}
T.~Xie and N.~Jiang.
\newblock Batch value-function approximation with only realizability.
\newblock In \emph{Proceedings of the 38th International Conference on Machine Learning}, volume 139 of \emph{Proceedings of Machine Learning Research}, pages 11404--11413, 2021.

\end{thebibliography}
